\documentclass[10pt,twocolumn,letterpaper]{article}

\usepackage{iccv}
\usepackage{times}
\usepackage{epsfig}
\usepackage{graphicx,booktabs}
\usepackage{amsmath,multirow}
\usepackage{amssymb}
\usepackage{amsmath,amsfonts,bm}

\def\eqref#1{equation~\ref{#1}}
\def\1{\bm{1}}

\DeclareMathAlphabet{\mathsfit}{\encodingdefault}{\sfdefault}{m}{sl}
\SetMathAlphabet{\mathsfit}{bold}{\encodingdefault}{\sfdefault}{bx}{n}

\usepackage{amsthm}
\usepackage{mathtools}

\usepackage{array}
\newcolumntype{x}[1]{>{\centering\arraybackslash\hspace{0pt}}p{#1}}

\newtheorem{definition}{Definition} 
\newtheorem{theorem}{Theorem}

\newcommand{\indep}{\rotatebox[origin=c]{90}{$\models$}}

\iccvfinalcopy % *** Uncomment this line for the final submission

\def\iccvPaperID{6403} % *** Enter the ICCV Paper ID here

\ificcvfinal\fi

\usepackage{authblk}

\usepackage[pagebackref=true,breaklinks=true,letterpaper=true,colorlinks,bookmarks=false]{hyperref}

\begin{document}

%%%%%%%%% TITLE
\title{Domain Intersection and Domain Difference\vspace{-1cm}}

\author[1]{Sagie Benaim}
\author[1]{Michael Khaitov}
\author[1]{Tomer Galanti}
\author[1,2]{Lior Wolf}
\affil[1]{School of Computer Science, Tel Aviv University} 
\affil[2]{Facebook AI Research}

\maketitle
\thispagestyle{empty}

%%%%%%%%% ABSTRACT
\begin{abstract}
We present a method for recovering the shared content between two visual domains as well as the content that is unique to each domain. This allows us to map from one domain to the other, in a way in which the content that is specific for the first domain is removed and the content that is specific for the second is imported from any image in the second domain. In addition, our method enables generation of images from the intersection of the two domains as well as their union, despite having no such samples during training. The method is shown analytically to contain all the sufficient and necessary constraints. It also outperforms the literature methods in an extensive set of experiments. 
Our code is available at \url{https://github.com/sagiebenaim/DomainIntersectionDifference}.
\end{abstract}

%%%%%%%%% BODY TEXT
\section{Introduction}

In unsupervised mapping between visual domains, the algorithm receives two unmatched sets of samples: one from domain $A$ and one from domain $B$. It then learns a mapping function that generates, for each sample $a$ in domain $A$, a matching sample in $B$. 

Without a supervision in the form of pairs of matched samples, the problem, like other unsupervised tasks, can be ambiguous~\cite{galanti2018the}. However, it is natural to expect that a pair of samples $(a,b)$, one from each domain, would be considered matching, if there is a significant amount of shared content between $a$ and $b$. The more content is shared, the stronger the link between the two samples.

Therefore, one can consider the intersection of two visual domains $A$ and $B$ as a domain that contains all of the information that is common to the two domains. This shared domain needs not be visual, and it can contain information that is encoded (latent information).

Turning our attention to the information that complements the shared information, each domain also has a separate, unshared part, which is domain-specific in the context of the two domains.

When mapping a sample $a$ from domain $A$ to $B$, we can, therefore, consider three types of information. The part of $a$ that is in the shared domain needs to remain fixed under the transformation. The part of $a$ that is specific to domain $A$ is discarded. Lastly, the part of the generated sample in $B$ that is specific to this domain is arbitrary.

While many unsupervised domain mapping methods do not specify the component that is specific to the second domain, some of the recent methods rely on a sample in $B$ to donate this information. Such methods are called guided image to image translation methods. The literature has two types of such methods: those that borrow the style from the image in $B$, assuming that the domain specific information is a type of visual style~\cite{munit,Lee_2018_ECCV}, and a recent method~\cite{press2018emerging} which assumes that domain $A$ is a subset of domain $B$, which does not contain any information that is not present in $B$. In both cases, these assumptions seem too strong. 

Our method is able to deal with the two separate domains in a symmetric way, without assuming that domain $B$ can contribute only a different style and without assuming that $A$ is a degenerate subset of $B$. The method employs a set of loss terms that lead, as our analysis shows, to a disentanglement between the three types of information that exist in the two domains. 

As a result, our method enables a level of control that is unprecedented in mapping image across domains. It allows us to take the specific part that belongs to domain $A$ from one image, the specific part of domain $B$ from another image, and the shared part from either image or from a third image. In addition, each of the three parts can be interpolated between different samples, and the domain specific parts can be eliminated altogether.

\subsection{Previous Work}

In image to image translation, the algorithm is provided with two independent datasets from two different domains. The goal is to learn a transformation of samples from the first domain to samples from the second domain. These transformations are often implemented by a deep neural network that has an encoder-decoder architecture. 

The early solutions to this problem assumed the existence of an invertible mapping $y$ from the first domain to the second domain. This mapping takes a sample $a$ in domain $A$ and maps it to an analog sample in domain $B$. The circularity based constraints by~\cite{CycleGAN2017,discogan,dualgan} are based on this assumption. In their work, they learn a mapping from one domain to the other and back in a way that returns the original sample, which requires no loss of information.  Nevertheless, this assumption fails to hold in a wide variety of domains. For example, in~\cite{CycleGAN2017} they show that when learning a mapping from images of zebras to images of horses, the stripes of the zebras are lost, which results in an ambiguity when mapping in the other direction. In our paper, we do not make assumptions of this kind. Instead, we take a very generic formulation that fits a wide variety of domains.

A few publications suggested learning many to many transformations. These papers include the augmentation based extension of CycleGAN~\cite{almahairi2018augmented}. In their generative model, they provide an additional random vector for each domain. Other methods such as the NAM method~\cite{nam} suggested non-adversarial training. In this model, the multiple solutions are obtained by different initializations. In our paper, multiple mappings are obtained by using a guide image. 

A powerful method for capturing the relations between the two domains is done by employing two different autoencoders that share many of their parameters~\cite{cogan,unit}. These constraints provide a shared representation of the two domains. Low-level image properties, such as color, texture and edges are domain-specific and are encoded and decoded separately. The higher level properties are shared between the two domains and are processed by the same layers in the autoencoders. In our paper, we employ a shared encoder for both domains to enforce a shared representation. Each domain has its own separate encoder to encode domain-specific content. Weight sharing is not used.

\noindent{bf Guided Translation}

The most relevant line of work learns a mapping between the two domains that takes two images as inputs: a source image $a$ from the first domain and a guide image $b$ from the second domain~\cite{munit,Lee_2018_ECCV,ma2018exemplar,press2018emerging}. The work of~\cite{munit,Lee_2018_ECCV,ma2018exemplar} employ a very narrow encoding for the domain specific content that is reflected by a low dimensional encoding. This enables them to only encode the style of the image in their domain specific encoder. However, since this encoding is very limited, it is impossible to capture the entire domain specific content. In our method, we do not rely on architectural restrictions to partition the information in the images into domain specific and common parts. Instead, our losses provide sufficient and necessary conditions for dividing the content into domain-specific and common contents in a principled way. Therefore, in our method we are able to capture a disentangled representation in which the common information in its entirety is encoded in the shared encoder and the complete domain-specific information is encoded in the separate encoders.

The very recent work of~\cite{press2018emerging} is probably the most similar to our work. In their paper, they tackle the problem where the source domain is a subset of the target domain (e.g., images of persons to images of persons with glasses). For such domains, a one-sided guided mapping from a source domain to a target domain is learned. For this purpose, they employ a common encoder, a separate encoder for the target domain and one decoder. To map between the source domain and the target domain, one applies the decoder on the common encoding of the source image and the separate encoding of the target domain. In their work, they are able to transfer the domain specific content for guiding the mapping from source to target. However, unlike our work, they are unable to handle the more general case, where both the source and the target domains have their own separate contents. This distinction is important, since even though they are able to provide content based guided mapping, they are limited to the case where the source domain behaves as a subset of the target domain. In our model, we are able to remove the content from the source images that is not present in the target images and not just to add content from images in the target domain.

Also related are several guided methods, which are trained in a supervised manner, i.e., the algorithm is provided with ground truth paired matches of images from domains $A$ and $B$. Unlike the earlier supervised one-to-one mapping methods, such as pix2pix~\cite{pix2pix}, these methods produce multiple outputs based on a guide image from the target domain. Examples include the Bicycle GAN by~\cite{zhu2017toward} and specific applications of the methods of~\cite{bao2017cvae,gonzalez2018image}.

In our method, disentanglement between the shared content and the two sources of domain-specific information emerge. Other work that relies on unsupervised or weakly supervised disentanglement, include the InfoGAN method~\cite{infogan}, which learns to disentangle a distribution to class-information and style, based on the structur of the data. \cite{lample2017fader,naama} learn a disentangled representation, by decreasing the class based information within it. We do not employ such class information. % The method of~\cite{cao2018dida} employs~\cite{naama} to perform guided image to image translation, however, it assumes the availability of class based information, which we do not.

\section{Problem Setup}\label{sec:setup}

We consider a framework with two different visual domains $A = (\mathcal{X}_A,\mathbb{P}_A)$ and $B = (\mathcal{X}_B,\mathbb{P}_B)$. Here, $\mathcal{X}_A, \mathcal{X}_B \subset \mathbb{R}^n$ are two sample spaces of visual images and $\mathbb{P}_A, \mathbb{P}_B$ are two distributions over them (resp.), i.e., the probability of $x \sim \mathbb{P}_A$ being $a$ is defined to be $\mathbb{P}_A[x=a]$. 

In this setting, we have two independent training datasets $\mathcal{S}_A = \{a_i\}^{m_1}_{i=1}$ and $\mathcal{S}_B = \{b_j\}^{m_2}_{j=1}$ sampled i.i.d from $\mathbb{P}_A$ and $\mathbb{P}_B$ (resp.). The set $\mathcal{S}_A$ (resp. $\mathcal{S}_B$) consists of training images from domain $A$ (resp. $B$).

Within a generative perspective, we assume that a sample $a \sim \mathbb{P}_A$ is distributed like $g(z_c,z_a,0)$ and a sample $b \sim \mathbb{P}_B$ is distributed like $g(z_c,0,z_b)$, where $z_c \sim \mathbb{P}_c$ and $z_a \sim \mathbb{P}^s_{A}$ and $z_b \sim \mathbb{P}^s_{B}$ are three latent variables. $z_c$ is considered a {\em shared content} between the two domains and $z_a$ and $z_b$ are {\em domain specific}.
The process is subject to the following independency relations. A sample $a$ from $A$ is generated such that, $z_c \indep z_a$ and a sample $b$ from $B$ is generated such that, $z_c \indep z_b$. The function $g$ takes a shared content $z_c \sim \mathbb{P}_c$ and a specific content $z_a \sim \mathbb{P}^s_{A}$ ( $z_b \sim \mathbb{P}^s_{B}$) and returns an image $g(z_c,z_a,0) \sim \mathbb{P}_A$ ( $g(z_c,0,z_b) \sim \mathbb{P}_B$). We assume that $g$ is invertible for both domains, i.e., there are functions $e^c$, $e^s_A$ and $e^s_B$, such that, for any sample $a \in \mathcal{X}_A$ and $b \in \mathcal{X}_B$, we have: 
\begin{equation}
a = g(e^c(a),e^s_A(a),0) \textnormal{ and } b = g(e^c(b),0,e^s_B(b))
\end{equation}
Here, $e^c$ denotes the function that takes a sample $a$ (or $b$) and returns its shared content, $e^s_A$ takes a sample $a$ and returns the specific content of $a$ and $e^s_B$ takes a sample $b$ and returns its specific content. As mentioned above, $e^c(a) \sim e^c(b)$, $e^c(a) \indep e^s_A(a)$ and $e^c(b)\indep e^s_B(b)$. For clarity, we note this is just a matter of modeling and we do not assume knowledge of the distributions of $z_c$, $z_a $ and $z_b$ nor $g$, $e^c$, $e^s_A$ and $e^s_B$. 

As a running example, let $A$ be a domain of images of non-smiling persons with glasses and $B$ a domain of images of smiling persons without glasses. In this case, $\mathcal{X}_A$ is a set of images of persons with glasses, $\mathcal{X}_B$ is a set of images of smiling persons. In addition, $\mathbb{P}_A$ are $\mathbb{P}_B$ are two distributions over these sets (resp.). The set $\mathcal{S}_A$ consists of $m_1$ training images of persons with glasses and $\mathcal{S}_B$ consists of $m_2$ training images of smiling persons. Here, the shared content $z_c$ between the two domains is an encoding of the identity and pose in an image (the image information excluding information about glasses or smile), $z_a$ is an encoding of glasses and $z_b$ is an encoding of a smile. The function $g$ is a generator that takes an encoding $z_c$ of a person and an encoding $z_a$ of glasses (or an encoding $z_b$ of a smile) and returns an image of the specified person with the specified glasses (or an image of the specified person with the specified smile). 

In this paper, we aim to learn an encoder-decoder model $G \circ E(x)$. Our encoder $E$ is composed of three parts: $E(x) := (E^c(x), E^s_A(x), E^s_B(x))$. Our goal is to make the first encoder, $E^c(x)$, capture the shared content between the two domains, $E^s_A(x)$, capture the content specific to images $a$ from $A$ and the third encoder, $E_B^s(x)$, capture the content present only in images $b$ from $B$. In addition, we want to make our generator $G$ be able to take $E^c(a)$ and $E^s_B(b)$ and return an image in $B$ that has the shared content of $a$ and the specific content of $b$ (and similarly in the opposite direction). Both the encoder and decoder are implemented with neural networks of fixed architectures. The specific architectural details are given in the appendix. % described in Sec.~\ref{sec:architecture}.

In the example above, for an image $a$ from $A$, we would like $E^c(a)$ to encode the person in the image $a$ (same for $b$ from domain $B$). We also want $E^s_A(a)$ to encode the glasses in the image $a$ and want $E^s_B(b)$ to encode the smile in the image $b$. We want $G$ to take $E^c(a)$ and $E^s_B(b)$ and to return an image of the person in $a$ without her glasses, but with the smile present in $b$. 

Formally, we would like to have the following two properties on the encoder-decoder:
\begin{equation}\label{eq:goal}
\begin{aligned}
& G(E^c(a),0,E^s_B(b)) \approx g(e^c(a),0,e^s_B(b)) \\
\textnormal{and } & G(E^c(b),E^s_A(a),0) \approx g(e^c(b),e^s_A(a),0)
\end{aligned}
\end{equation}
Here, $0$ in the first equation stands for zeroing the coordinates of $E^s_A(x)$ in the encoder $E(x)$ (similarly for the second equation). 

Since we do not have any paired matches of any of the forms: $(a,b) \mapsto g(e^c(a),0,e^s_B(b))$ or $(a,b) \mapsto g(e^c(b),e^s_A(a),0)$ (the left-hand-side is a pair of images and the right-hand-side is a single image) it is unclear how to make the encoder-decoder $G \circ E$ satisfy Eq.~\ref{eq:goal}. Concretely, since we are only provided with unmatched images of persons with glasses and images of smiling persons, it is not obvious how to learn a mapping that takes an image of a person with glasses and an image of a smiling person and returns an image of the first person without the glasses, but with the smile from the second image. We present a set of training constraints that are both necessary and sufficient for performing this training.%For this purpose we present a more sophisticated yet simple algorithm for solving this problem.

\section{Method}\label{sec:method}

In Sec.~\ref{sec:setup} we defined the different components of the proposed framework. In addition, we explained that it is not obvious how to solve Eq.~\ref{eq:goal} without any supervised data. In this section,  we explain our method for solving this problem in the proposed unsupervised setting. 

As mentioned, our method consists of three encoders, $E^{c}$, $E^{s}_{A}$ and $E^{s}_{B}$ and a decoder $G$. $E^{c}$ encodes the information content common to $\mathbb{P}_A$ and $\mathbb{P}_B$. The two other encoders, $E^{s}_{A}$ and $E^{s}_{B}$, encode the information content specific to samples of $\mathbb{P}_A$ and $\mathbb{P}_B$ (resp.). To solve this, we use three types of losses: ``zero'', adversarial, and reconstruction.

\subsection{Zero Loss}

We would like to enforce $E^s_A$ ($E^s_B$) to capture information relevant to domain $A$ only. To do so we force $E^s_A$ ($E^s_B$) to be 0 on samples in $B$ ($A$):
\begin{align}
\mathcal{L}_{zero}^A &:= \frac{1}{m_2} \sum^{m_2}_{j=1} \| E^s_A(b_j) \|_1 \\
\mathcal{L}_{zero}^B &:= \frac{1}{m_1} \sum^{m_1}_{i=1} \| E^s_B(a_i) \|_1 \\
\mathcal{L}_{zero} &:= \mathcal{L}_{zero}^A + \mathcal{L}_{zero}^B
\end{align}

As illustrated in Fig~\ref{fig:diagram}(a), if $A$ is the domain of persons with glasses and $B$ is that of smiling persons, then this loss ensures that $E^s_A$ ($E^s_B$) will not capture any information about the face or smile (face or glasses).

\subsection{Adversarial Loss}

We would like to capture the fact that the common encoder, $E^c$, does not capture more information than necessary. In the running example, we would like $E^c$ not to capture information about smile or glasses. This is illustrated in Fig~\ref{fig:diagram}(c). To do so, we use an adversarial loss to ensure that the distribution $\mathbb{P}_{E^c(A)}$ of $E^c(a)$ equals the distribution $\mathbb{P}_{E^c(B)}$ of $E^c(b)$. The loss $\mathcal{L}_{adv}$ is given by:
\begin{small}
\begin{align}
 \frac{1}{m_1} \sum^{m_1}_{i=1} l(d(E^c(a_i)),1)+ \frac{1}{m_2}\sum^{m_2}_{j=1} l(d(E^c(b_j)),1)
\end{align}
\end{small}
$d$ is a discriminator network, and  $l(p,q) = -(q\log(p)+(1-q)\log(1-p))$ is the binary cross entropy loss for $p\in\left[0,1\right]$ and $q\in\{0, 1\}$.
The network $d$ minimizes the loss:
\begin{small}
\begin{align}
    \mathcal{L}_d &:= \frac{1}{m_1} \sum^{m_1}_{i=1} l(d(E^c(a_i)),0)+ \frac{1}{m_2}\sum^{m_2}_{j=1} l(d(E^c(b_j)),1)
\end{align}
\end{small}
The discriminator $d$ attempts to separate between the distributions $\mathbb{P}_{E^c(A)}$ and $\mathbb{P}_{E^c(B)}$ of $E^c(a)$ and $E^c(b)$ (resp.), by classifying samples of the former as $0$ and the samples of the latter as $1$, whereas the encoder tries to fool the discriminator, hence forcing both distributions to match.

Referring back to our running example, this loss is a confusion term that ensures that the encoding by $E^c$ of face images do not contain information on whether the person is smiling and on whether the person wears glasses.

\subsection{Reconstruction Loss}

Both the zero loss and the adversarial loss ensure that no encoder encodes more information than needed. However, we need to also ensure that all the needed information is encoded. In particular, $E^s_A$ ($E^s_B$) should capture all the separate information in $A$ ($B$). $E^c$ should capture all the common information between $A$  and $B$, but not less. To do so, we force the information in $E^s_A(a)$ and $E^c(a)$ to be sufficient to reconstruct $a$, and similarly that the information in $E^s_B(b)$ and $E^c(b)$ is sufficient to reconstruct $b$. Specifically, we have:
\begin{align}
\mathcal{L}_{recon}^A &:= \frac{1}{m_1} \sum^{m_1}_{i=1} \| G(E^c(a_i),E^s_A(a_i),0) - a_i \|_1 \\
\mathcal{L}_{recon}^B &:= \frac{1}{m_2} \sum^{m_2}_{j=1} \| G(E^c(b_i),0,E^s_B(b_j)) - b_j \|_1  \\
\mathcal{L}_{recon} &:= \mathcal{L}_{recon}^A + \mathcal{L}_{recon}^B 
\end{align}

\subsection{Full Objective}

For the full objective, $E_c$, $E^s_A$, $E^s_B$ and $G$ jointly minimize the following objective:
\begin{align}
\mathcal{L} =  \mathcal{L}_{zero} + \lambda_1\mathcal{L}_{adv} + \lambda_2\mathcal{L}_{recon}
\end{align}
Where $\lambda_1$  and $\lambda_2$ are positive constants. The discriminator $d$ minimizes the loss $\mathcal{L}_d$ concurrently. The full description of the architecture employed for the encoders, generator and discriminator is given in the appendix.

\begin{figure*}[t]
    \centering
    \includegraphics[scale=0.25]{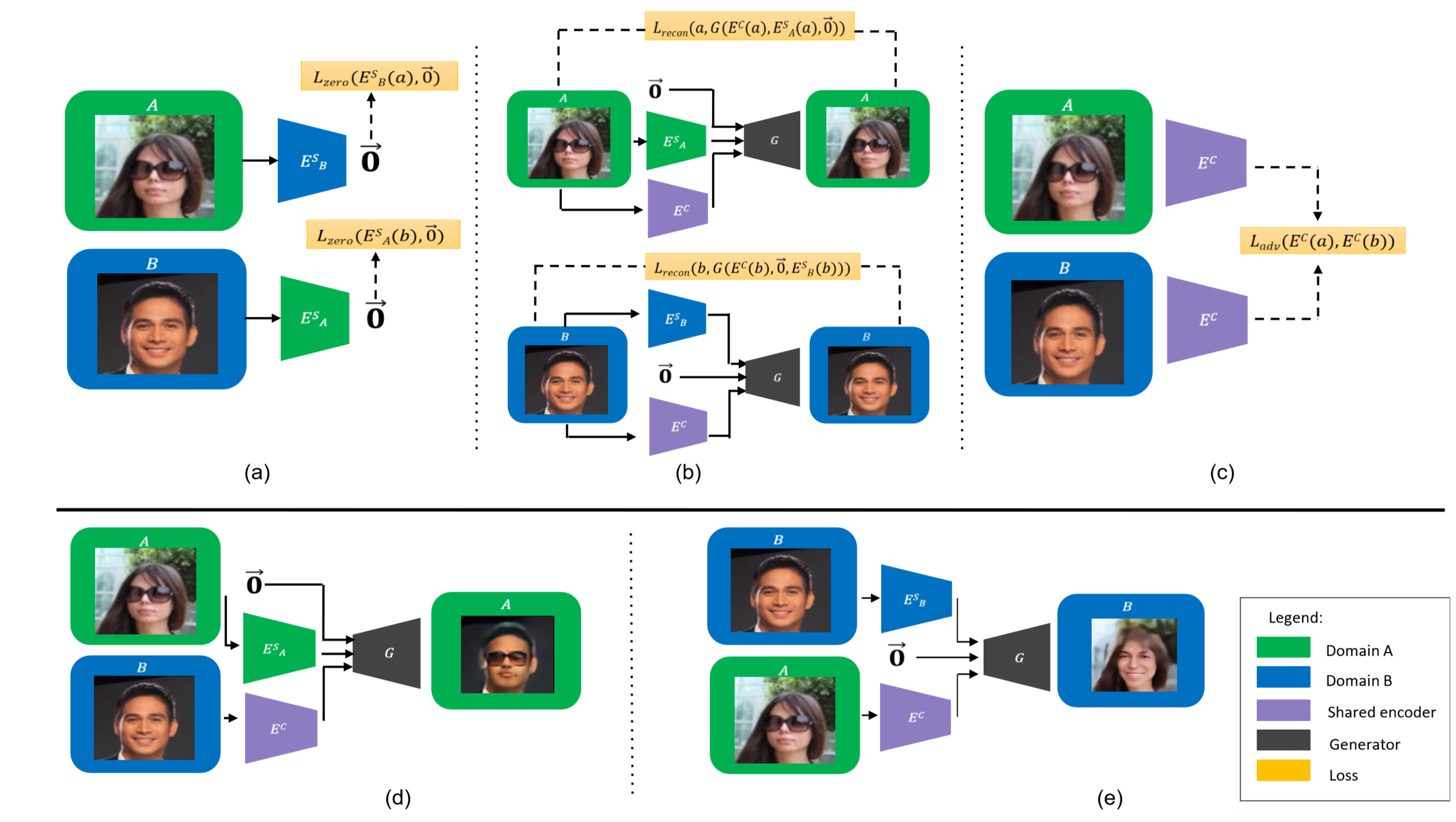}
    \caption{Illustration of the train and inference stages. The losses are illustrated in (a), (b) and (c) and the guided mappings are illustrated in (d) and (e). (a) Illustration of the zero loss. Encoding images from domain $A$ (illustrated in green) with domain's $B$ separate encoder should result in a zero vector, encoding no information about the image (and vice versa). (b) Illustration of the reconstruction loss. Given $a$'s separate encoding (illustrated in green), for example glasses, and its common encoding (illustrated in purple), for example all other facial features, it should be possible to reconstruct $a$ (same for domain $B$). (c) Illustration of the adversarial loss. The distribution of the common encoding from domain $A$ and domain $B$ (face features) should be the same. To enforce this, an adversarial loss is used. (d) Constructing new images. At inference time we can encode domain's $B$ image $b$ using its separate encoder to get its smile, encode the common domain $A$'s image $a$ (face features without glasses) and generate an image similar to $a$, but without glasses and with $b$'s smile. (e) Similarly to (d), we can generate an image similar to $a$ but with the smile removed and glasses of $b$ added.}
    \label{fig:diagram}
\end{figure*}

\section{Theoretical Analysis}\label{sec:analysis}

We provide an informal theoretical analysis for the success of the proposed method. For the formal version, please refer to the appendix. 

In Sec.~\ref{sec:setup} we represented our random variable $a \sim \mathbb{P}_A$ and $b \sim \mathbb{P}_B$ in the following forms $a = g(e^c(a),e^s_A(a),0)$ and $b = g(e^c(b),0,e^s_B(b))$, where $e^c(a) \indep e^s_A(a)$, $e^c(b) \indep e^s_B(b)$ and $g$ is an invertible function. 

Before we present our theorem regarding emerging disentanglement between the learned encoders, we provide a necessary definition of an intersection. An intersection of two independent random variables $a$ and $b$ are two representations $a = g(e^c(a),e^s_A(a),0)$ and $b = g(e^c(b),0,e^s_B(b))$, such that, the common encoding $e^c(a) \sim e^c(b)$ has the largest amount of information (measured by entropy $H$). For example, let us consider the case in which domain $A$ consists of images of persons wearing glasses and domain $B$ consists images of smiling persons. In this case, we can encode the samples of $A$ into (i) an identity and pose encoding and (ii) a glasses encoding. Similarly, we can encode the samples of $B$ into the first encoding of domain $A$ and the encoding of the smile. This representation forms an intersection, since we cannot transfer common information from the glasses and the smile into the common part. 

\begin{definition}[Intersection]\label{def:inter} We say that the two representations $a = g(e^c(a),e^s_A(a),0)$ and $b = g(e^c(b),0,e^s_B(b))$ form an intersection between $a$ and $b$, if for any other representation $a = \hat{g}(\hat{e}^c(a),\hat{e}^s_A(a),0)$ and $b = \hat{g}(\hat{e}^c(b),0,\hat{e}^s_B(b))$, such that, $\hat{g}$ is invertible and $\hat{e}^c(a) \sim \hat{e}^c(b)$, we have: $H(\hat{e}^c(a)) \leq H(e^c(a))$.
\end{definition}

The following theorem shows that under reasonable conditions, by minimizing the proposed losses, we obtain a disentangled representation.

\begin{theorem}[Informal]
In the setting of Sec.~\ref{sec:setup}. Let $a \sim \mathbb{P}_A$ and $b \sim \mathbb{P}_B$ be two random variables. Assume that the representations $g(e^c(a),e^s_A(a),0)$ and $g(e^c(b),0,e^s_B(b))$ form an intersection between $a$ and $b$. Assume that we cannot recover the sample $a$ from the separate encoding $E^s_A(a)$. Assume that the reconstruction and adversarial losses are minimized by $E^c, E^s_A, E^s_B$ and $G$. Then, we obtain a disentanglement between $E^c(a)$ and $E^s_A(a)$, such that, $E^c(a)$ captures the information of $e^c(a)$ and $E^s_A(a)$ captures the information of $e^s_A(a)$.
\end{theorem}

The theorem makes three types of assumptions. The first type is about the modeling of the data, i.e., that it follows the problem definition in Sec.~\ref{sec:setup} and that the shared part of the model ($e^c$) is an intersection of the two domains. The second assumption is regarding the separate encoder we learn ($E^s_A$) and it states that one cannot reconstruct $a$ from $E^s_A(a)$. The last group of assumptions concerns the losses, which we minimize in our algorithm. 

The conclusion of this theorem is that under the proposed assumptions, (i) the common $E^c(a)$ and separate $E^s_A(a)$ parts are independent, (ii) the common part $E^c(a)$ captures the information in the underlying $e^c(a)$, and (iii) the separate part $E^s_A(a)$ captures the information in $e^s_A(a)$. Therefore, we obtain the desired encoding of domain $A$. By symmetric arguments, we arrive at the same conclusions for $E^c(b)$ and $E^s_B(b)$.

\section{Experiments}

To evaluate our method, we consider the celebA~\cite{celeba} dataset, which consists of celebrity face images with different attributes. We consider the smile, glasses, facial hair, male, female, blond and black hair attributes. Each of these attributes can be used as domain $A$ or $B$ symmetrically. 

\subsection{Guided translation between domains}

\begin{table*}[t]
\centering
\begin{small}
  \begin{tabular}{lcccccc}
    \toprule
     &  Smile To & Glasses & Facial Hair & Smile To & Facial Hair & Glasses To \\ 
     &  Glasses & To Smile & To Smile & Facial Hair & To Glasses& Facial Hair\\
    \midrule
  	% Fader networks~\cite{lample2017fader} & 0.066 & 0.064 & 0.182 \\ 
  	Fader networks~\cite{lample2017fader} & 76.8\% & 97.3\% & 95.4\% & 84.2\% & 77.8 \% & 85.2\% \\ 
    Guided content transfer~\cite{press2018emerging} & 45.8\% & 92.7\% & 85.6\% & 85.1\% & 38.6\% & 82.2\% \\
    MUNIT~\cite{munit} & 7.3\% & 9.2\% & 9.3\% & 8.4\% & 7.3\%  & 8.5\% \\
    DRIT~\cite{Lee_2018_ECCV} & 8.5\% & 6.3\% & 6.3\% & 10.3\% & 8.6\%  & 10.1\% \\
    Ours & 	 91.8\% & 99.3\% & 93.7\% & 87.1\% & 93.1\% & 97.2\% \\ %
    \bottomrule
\end{tabular}
\end{small}
\smallskip
  \caption{We pretrain a classifier to distinguish between samples in $A$ (e.g. images of persons with glasses) and samples in $B$ (e.g. images of persons with smile). We then sample $a \in A$, $b \in B$ from the test samples and check the membership of the generated image $G(E^c(b),E_A^s(a), 0))$ in $A$. Similarly, in the reverse direction, we check the membership of $G(E^c(a), 0, E_B^s(b))$ in $B$. }
  \label{tab:classifier}
\end{table*}

\begin{table*}[t]
\centering
\begin{small}
  \begin{tabular}{lcccccc}
    \toprule
    &  Smile To & Glasses & Facial Hair & Smile To & Facial Hair & Glasses To \\ 
     &  Glasses & To Smile & To Smile & Facial Hair & To Glasses& Facial Hair\\
    \midrule
  	% Fader networks~\cite{lample2017fader} & 0.066 & 0.064 & 0.182 \\ 
    Question (1) ours & 4.74 $\pm 0.13$ & 4.30 $\pm 0.21$ & 4.26 $\pm 0.20$ & 4.30 $\pm 0.15$& 4.18 $\pm 0.17$ & 4.50 $\pm 0.18$ \\ %
    Question (2) ours & 3.92 $\pm 0.16$ & 4.45 $\pm 0.12$ & 4.03 $\pm 0.15$ & 3.34 $\pm 0.17$& 3.85 $\pm 0.20$& 3.95 $\pm 0.22$\\ 
    Question (3) ours & 3.95 $\pm 0.23$& 3.20 $\pm 0.24$& 3.24 $\pm 0.25$& 3.22 $\pm 0.27$& 3.49 $\pm 0.22$& 3.39 $\pm 0.23$\\
    \midrule
    Question (1) for \cite{press2018emerging} & 3.67 $\pm 0.17$ & 4.16 $\pm 0.18$& 3.39 $\pm 0.19$& 3.34 $\pm 0.13$& 4.24 $\pm 0.12$& 3.15 $\pm 0.15$\\ %
    Question (2) for \cite{press2018emerging} & 1.87 $\pm 0.35$ & 4.42 $\pm 0.22$& 3.00 $\pm 0.32$& 2.67 $\pm 0.33$ & 2.20 $\pm 0.42$& 3.30 $\pm 0.22$\\
    Question (3) for \cite{press2018emerging} & 3.95 $\pm 0.15$& 2.93 $\pm 0.22$ & 3.37 $\pm 0.25$& 3.40 $\pm 0.27$ & 3.43 $\pm 0.28$ & 3.75 $\pm 0.20$\\
    \bottomrule
\end{tabular}
\end{small}
\smallskip
  \caption{Given $20$ randomly selected images $a \in A$ and $b \in B$, we consider the generated image  $G(E^c(a), 0, E_B^s(b)))$ and ask if (1) a's separate part is removed (2) b's separate part is added (3) a's common part is preserved (similarly in the reverse direction). Mean opinion scores in the range of 1 to 5 are reported, where higher is better.   }
    \label{tab:user_study}
\end{table*}

% \begin{table*}[t]
% \centering
%   \begin{tabular}{lcccccc}
%     \toprule
%     &  Smile To & Glasses & Facial Hair & Smile To & Facial Hair & Glasses To \\ 
%      &  Glasses & To Smile & To Smile & Facial Hair & To Glasses& Facial Hair\\
%     \midrule
%   	% Fader networks~\cite{lample2017fader} & 0.066 & 0.064 & 0.182 \\ 
%     Question (1) ours & 4.74	& 4.30 & 4.26 & 4.30 & 4.18 & 4.50 \\ %
%     Question (2) ours & 3.92 & 4.45 & 4.03 & 3.34 & 3.85 & 3.95 \\ 
%     Question (3) ours & 3.95 & 3.20 & 3.24 & 3.22 & 3.49 & 3.39 \\
%     \midrule
%     Question (1) for \cite{press2018emerging} & 3.67 & 4.16 & 3.39 & 3.34 & 4.24 & 3.15 \\ %
%     Question (2) for \cite{press2018emerging} & 1.87 & 4.42 & 3.00 & 2.67 & 2.20 & 3.30 \\ 
%     Questino (3) for \cite{press2018emerging} & 3.95 & 2.93 & 3.37 & 3.40 & 3.43 & 3.75 \\
%     \bottomrule
% \end{tabular}
% \smallskip
% \smallskip
%   \caption{Given $20$ randomly selected images $a \in A$ and $b \in B$, we consider the generated image  $G(E^c(a), 0, E_B^s(b)))$ and ask if (1) a's separate part is removed (2) b's separate part is added (3) a's common part is preserved (similarly in the reverse direction). Mean opinion scores in the range of 1 to 5 are reported, where higher is better.   }
%     \label{tab:user_study}
% \end{table*}

In Fig.~\ref{fig:glasses_sheet}, we consider $A$ to be the domain of images of smiling persons and $B$ to be the domain of images of persons with glasses. Given a sample $a \in A$ (top row) and a sample $b \in B$ (left column), each image constructed is of the form $G(E^c(a), 0, E^s_B(b))$. The common features of image $a$ (its identity) are preserved, the smile is removed, and the glasses of $b$ are added (the guide image). The reverse direction, as well as other cross domain translations, are depicted in the appendix. 

In order to evaluate the success of the translation numerically, we pretrain a classifier to distinguish between images from domain $A$ and domain $B$. If the specific part of the domain $A$ was successfully removed (for example, smile), and the specific part of domain $B$ was successfully added (for example, glasses), then the classifier should classify the translated image as a domain $B$ image. Tab.~\ref{tab:classifier} shows the success of our method in this case, in comparison to the baseline methods of~\cite{press2018emerging,lample2017fader,munit,Lee_2018_ECCV}, which are much less successful in switching attributes. Specifically: (i) MUNIT~\cite{munit} and DRIT~\cite{Lee_2018_ECCV} only change style, but the content is unchanged, (ii) Fader networks~\cite{lample2017fader} translated between the domains, in a less convincing way, that also ignores the guide image, and (iii) The method of Press et al.~\cite{press2018emerging} adds the element of the target domain, but fails to remove the content of the source domain.

By conducting a user study, we evaluate the ability to (a) remove the specific attribute of domain $A$ (b) add the specific attribute of domain $B$, and (c) preserve the identity of the image encoded in the common encoder.  To do so, given an image $a$ from domain $A$ and an image $b$ from domain $B$, we present the user with two images $a \in A$, $b \in B$ and the generated image $G(E^c(a), 0, E^s_B(b))$ (or $G(E^c(b), E^s_A(a), 0)$ for the reverse direction), and ask the following three questions: 1. Is the specific attribute of $A$ (e.g smile) removed? 2. Is the guided image $b$ specific attribute (e.g glasses) added? 3. Is the identify of $a$'s image preserved (that is, is the common attribute from $a$ still present in the image)? Mean Opinion Score on the scale of $1$ to $5$, are collected for $20$ randomly selected test images in $A$ and $B$ by $20$ different users is reported in Tab.~\ref{tab:user_study}.  
For most translations, the ability to remove $A$'s specific attribute and add $B$'s specific attribute is significantly better than that of~\cite{press2018emerging}, while the ability to preserve the identity of $a$ is on-par with \cite{press2018emerging}. The Fader networks~\cite{lample2017fader} provides a generic (unguided) cross domain translation, and MUNIT~\cite{munit} transfers style and not content and were therefore not included in the user study. See the appendix for the results obtained by these methods.

\subsection{Linearity of latent space}

\begin{figure}
\centering
  \includegraphics[width=0.95\linewidth, clip]{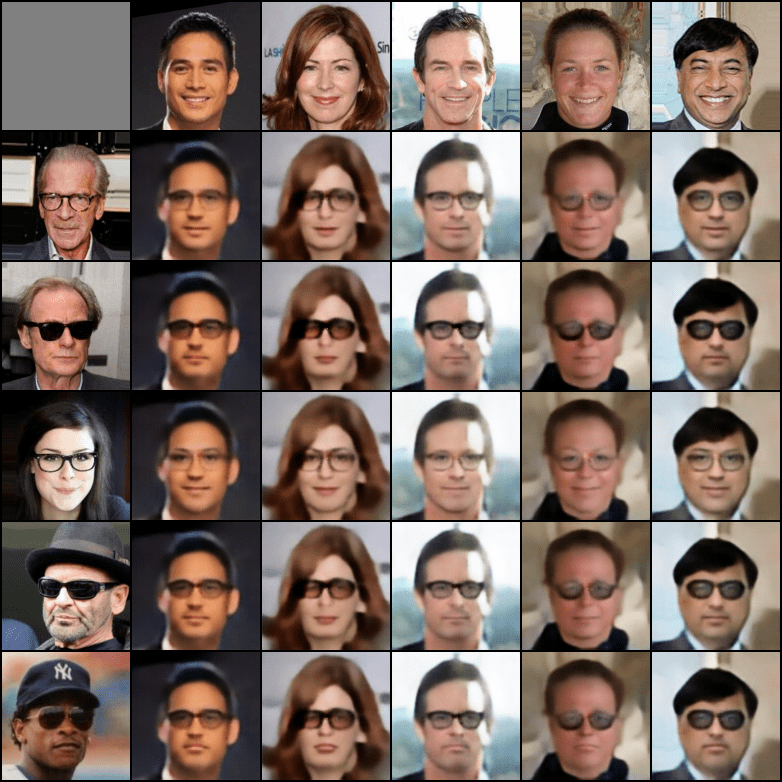} \\ 
   \caption{Images $a \in A$ are in the top row and $b \in B$ in the left column. The images constructed are $G(E^c(a),0, E_B^s(b)))$, consisting of the common parts of $a$  and separate part of $b$ (smile is removed and glasses added).  }
  \label{fig:glasses_sheet}
  ~\\
%\end{figure}
%\begin{figure}
\centering
  \includegraphics[width=0.95\linewidth, clip]{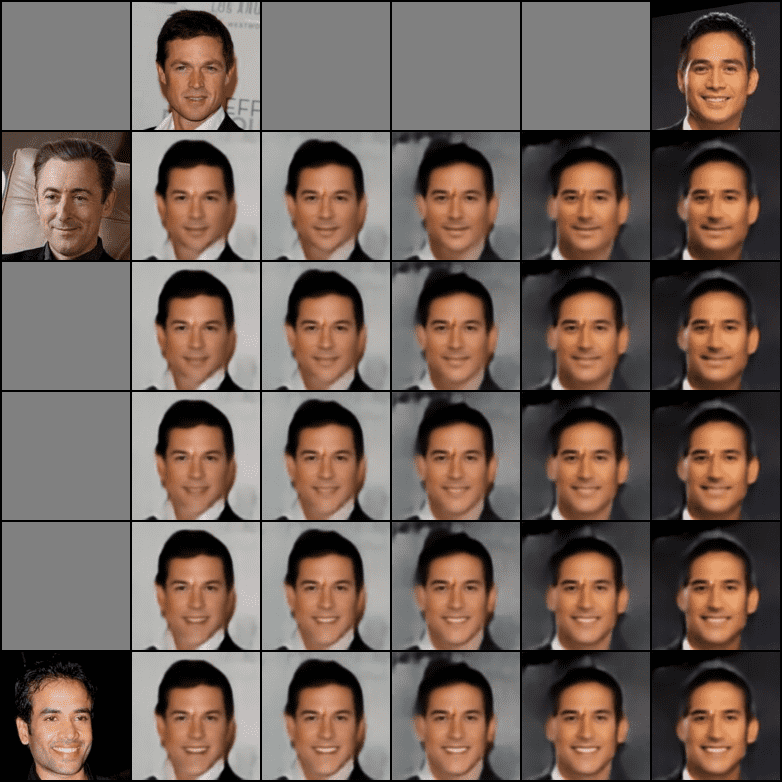} \\ 
  \caption{Interpolation in the latent space of domain $A$ (smiling). We linearly interpolate between the common encoding of the two images in the top row going left to right. Concurrently, we linearly interpolate between the separate encoding of the two images in the left column going top to bottom.  }
  \label{fig:glasses_interpolation_sep_a}
\end{figure}

\begin{figure}
\centering
  \includegraphics[width=0.95\linewidth, clip]{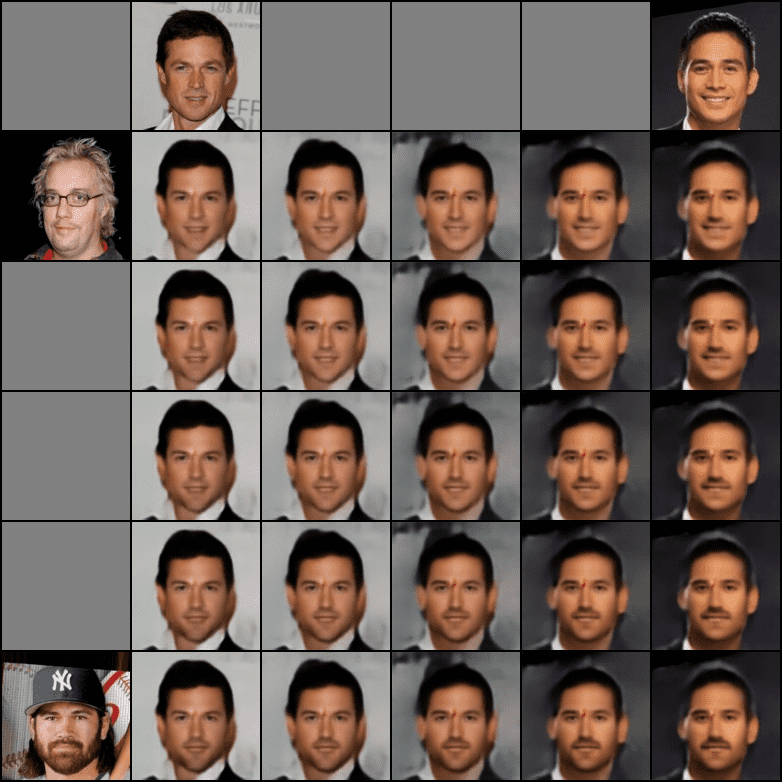} \\
  \caption{Interpolation in the latent space of domains $A$ (smiling) and $B$ (facial hair). We interpolate the common encoding of the two images from domain $A$ in the top row. Concurrently, we linearly interpolate between the separate encoding of the two images from domain $B$ in the left column.}
  \label{fig:glasses_interpolation_sep_b}
  ~\\
%\end{figure}
%\begin{figure}
\centering
  \includegraphics[width=0.95\linewidth, clip]{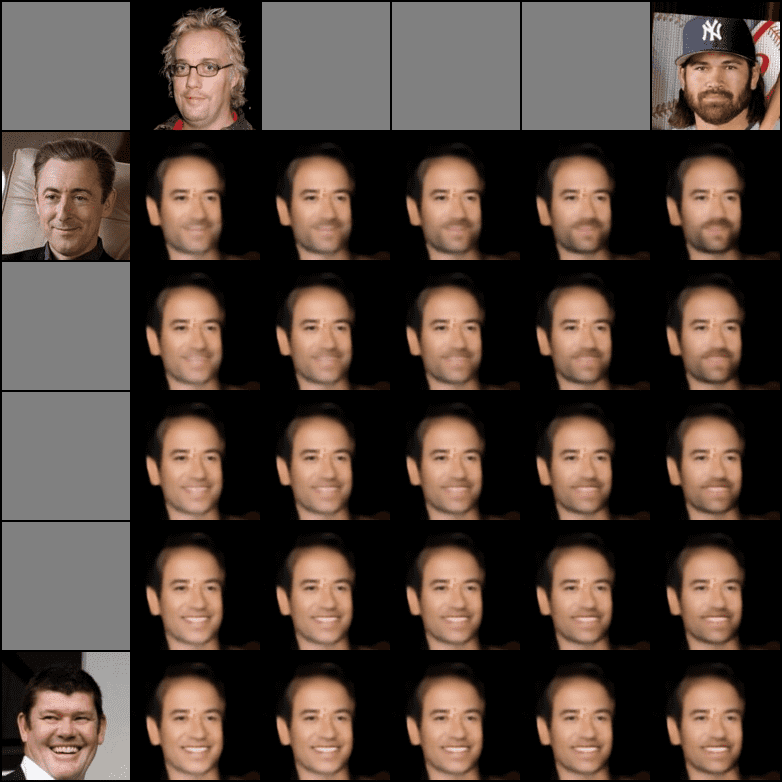}
  \caption{Interpolation domains $A$ (smiling) and $B$ (facial hair). Fixing the common encoding to randomly chosen image, we interpolate between $A$'s separate encoding of the two images in the top row. Concurrently, we interpolate between $B$'s separate encoding of the two images in the left column.}
  \label{fig:glasses_interpolation_sep_c}
  \vspace{-12pt}
\end{figure}
 We evaluate the linearity of the latent representation of $A$'s separate encoder, $B$'s separate encoder and the common encoder. In this case, $A$ serves as the domain of images of smiling persons and $B$ of images of persons with facial hair. In Fig.~\ref{fig:glasses_interpolation_sep_a} the generated images take the form $G(com, a, 0)$, where $com = \alpha E^c(a_1) + (1-\alpha)E^c(a_2)$ and $a = \beta E^s_A(a_3) + (1{-\beta})E^s_A(a_4)$. $\alpha$ ranges between $0$ and $1$, going left to right and $\beta$ ranges from $0$ to $1$, going from top to bottom. $a_1, a_2, a_3, a_4$ are images from domain $A$ (smiling persons), given in the top row and left column. We observe that the latent representations produced by $A$'s separate encoder and the common encoder are linear. 

Similarly, in Fig.~\ref{fig:glasses_interpolation_sep_b} we evaluate the linear separability of $B$'s separate encoder. Generated images take the form $G(com, 0, b)$, where $com = \alpha E^c(a_1) + (1-\alpha)E^c(a_2)$ and $b = \beta E^s_B(b_1) + (1{-\beta)}E^s_B(b_2)$. $\alpha$ ranges between $0$ and $1$, going left to right, and $\beta$ ranges between $0$ and $1$, going from top to bottom. $a_1, a_2$ are images from domain $A$ given in the top row and $b_1, b_2$ are images from domain $B$ in the left column.

Lastly, in Fig.~\ref{fig:glasses_interpolation_sep_c}, we fix the common part from some image $c$, and evaluate the linearity of both separate encoders applied together. Generated images take the form $G(com, a, b)$, where $com = E^c(c)$ and $a = \alpha E^s_A(a_1) + (1-\alpha)E^s_A(a_2)$  and $b = \beta E^s_B(b_1) + (1-\beta)E^s_B(b_2)$. $\alpha$ ranges from $0$ to $1$ going left to right and $\beta$ ranges from $0$ to $1$ going from top to bottom. $c$ is a fixed image in $A$, while $a_1, a_2$ are images from domain $A$ given in the top row and $b_1, b_2$ are images from domain $B$ in the left column. 

Note that in this last case, we generate images from the union domain, i.e., create images that have, in addition to the common information, both the added content of $A$ and of $B$. The method also allows us to consider the intersection domain. In the depicted example, domain $A$ includes images of persons with glasses and $B$ includes images of smiling persons. The intersection of $A$ and $B$ consists of images of non-smiling persons (without glasses). Having never seen such images in the training set, our method now allows us to generate images from this distribution. This is illustrated in Fig.~\ref{fig:glasses_removal}. To do so, the generated image is of the form $G(E^c(x), 0, 0)$, where $x$ is a member of $A$ or $B$.

\subsection{Unsupervised Domain Adaptation}

To evaluate the disentangled representation, we perform unsupervised domain adaptation experiments translating from MNIST to SVHN. In this problem, the underlying framework is used to translate from MNIST to SVHN and a pretrained classifier is used to evaluate the percentage of images mapped to the same label in the target domain. In our case, given an MNIST digit $a$, we randomly sample an SVHN digit $b$ and consider the translation to SVHN as $G(E^c(a), 0, E^s_B(b))$. 
% The results are reported in Tab.~\ref{tab:uda}. 
In the MNIST to SVHN direction our method has $61.0\%$ accuracy beating Vae-NAM~\cite{vaenam} (51.7\%), NAM~\cite{nam} (31.9\%), DistanceGAN~\cite{distgan} (27.8\%) and CycleGAN~\cite{CycleGAN2017} (17.7\%). In the reverse direction it has $41.0\%$ accuracy beating Vae-NAM (37.4\%), NAM (33.3\%), DistanceGAN (27.8\%)  and CycleGAN (26.1\%).

% \begin{table}[t]
% \begin{center}
%   \begin{tabular}{lcccc}
%     \toprule
%   	% Fader networks~\cite{lample2017fader} & 0.066 & 0.064 & 0.182 \\ 
%   	 & CycleGAN  & Vae-NAM~\cite{vaenam} & Ours\\ %
%   	\midrule
%   	MNIST $\rightarrow$ SVHN& 17.7 \% & 51.7 & 61.0 \% \\
%   	SVHN $\rightarrow$ MNIST& 26.1 \% & 37.4 & 41.0 \% \\
% %   	No  zero loss & 85.4\% & 97.8\% \\
% %     No adversarial loss & 64,5\% & 79.3\% \\
% %     No reconstruction loss & 50.0\% & 50.0\% \\ 
%     \bottomrule
% \end{tabular}
% \end{center}
% \caption{Results for unsupervised domain adaptation from MNIST to SVHN and vice-versa. }
%   \label{tab:uda}
% \end{table}

\subsection{Ablation study}

We consider the formulation of our objective with each of the three parts missing: the adversarial loss, the zero loss and the reconstruction loss. We conduct an ablation study in the case of $A$ being images of smiling persons and $B$ is the domain of images of persons with glasses. The results, which appear in Tab.~\ref{tab:ablation} and shown visually in the appendix, indicate that when the reconstruction loss is missing, the method is unable to generate realistic looking images. %, and noise is generated as output. %This is as in this case the generator is untrained.  
In the case of no adversarial loss, the method is able to remove the smile but unable to add glasses from $b$. Without the adversarial loss, the common encoder can contain information specific to the domain, such as glasses, and so there would be no need to encode it in the separate encoder. Lastly, without the zero loss, the translation is slightly worse but still succeeds to a large extent. As shown in our analysis, the enforcing of the zero loss is not required to achieve the desired disentanglement effect.

\begin{table}[t]
\begin{small}
\begin{center}
\begin{tabular}{lcc}
    \toprule
  	% Fader networks~\cite{lample2017fader} & 0.066 & 0.064 & 0.182 \\ 
  	All Losses & 91.8\%	  &  99.3\% \\ %
  	\midrule
  	No  zero loss & 85.4\% & 97.8\% \\
    No adversarial loss & 64,5\% & 79.3\% \\
    No reconstruction loss & 50.0\% & 50.0\% \\ 
    \bottomrule
\end{tabular}
\end{center}
\end{small}
\caption{An ablation study for the case where $A$ is persons with glasses and $B$ is smiling persons. We consider the same setting as Tab~\ref{tab:classifier}, and consider the effect of removing each loss on the classification loss. The left column is for the Smile To Glasses task and the right column is for the Glasses To Smile task.}
  \label{tab:ablation}
\end{table}

\begin{figure}
\vspace{-.3cm}
\centering
  \includegraphics[width=0.95\linewidth, clip]{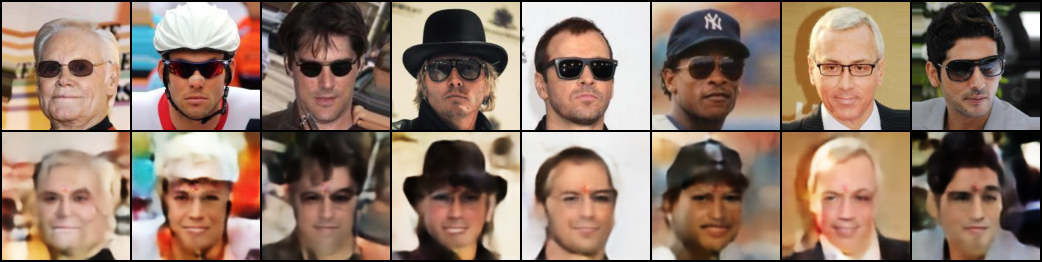} \\ 
  \caption{Generating images from the intersection of $A$ and $B$. (top) image from $A$. (bottom) mapping to the intersection domain.}
  \label{fig:glasses_removal}
\end{figure}

\section{Conclusions}

The field of unsupervised learning presents new problems that go beyond the classical methods of clustering or density estimation. The problem of unsupervised cross-domain translation was not considered solvable up to a few years ago. Recently, a set of guided translation problems have emerged, in which one maps between domains based on the features of a reference image in the target domain. While the literature methods treat the two domains in an asymmetric way (one domain donates style and another content, or one domain is a subset of the second), our work is the first to treat the domains in a symmetric way. 

Our work also presents the first method that is able to create images that have guided elements from two different domains, extracted from donor images $a$ and $b$ (one from each domain) and overlaid on a third image (taken from either domains) that donates the shared content.

The method we propose is shown to provide a sufficient set of constraints in order to support this conversion. It does not employ GANs in the visual domains, or cycles of any sort. The constraints are simple structural and reconstruction constraints, with the addition of a domain confusion loss, applied in the shared latent space.

Our experiments show that the new method provides superior results for the symmetrical guided domain problem  in comparison to the literature methods. Going forward, the ability to intersect domains (creating a domain that is orthogonal to the specific parts of the two domains), construct their union (combining both specific parts and the shared part), and consider the difference between the two, could lead to the ability to perform domain arithmetics and construct complex visual domains by combining, in a very flexible way, an unlimited number of domains.

\subsection*{Acknowledgements}
This project has received funding from the European Research Council (ERC) under the European Unions Horizon 2020 research and innovation programme (grant ERC CoG
725974). The contribution of Sagie Benaim is part of a
Ph.D. thesis research conducted at Tel Aviv University.

{\small
\bibliographystyle{ieee_fullname}
\bibliography{main}
}

\appendix

\section{Additional Guided Translation Results}

We provide the reverse translation to that given in Fig.~2 of the main report as well as additional cross domain translations in Fig.~\ref{fig:glasses_smiling}, \ref{fig:beard_smiling}, \ref{fig:smiling_beard}, \ref{fig:glasses_beard}, \ref{fig:beard_glasses}, \ref{fig:male_female}, \ref{fig:female_male}, \ref{fig:blond_black} and \ref{fig:black_blond}. 

Both forward and reverse directions are trained simultaneously using the same model as our model is symmetric. In the reverse direction, Given a sample $b \in B$ (top row) and a sample $a \in A$ (left column), each image constructed is of the form $G(E^c(b), E^s_A(a), 0)$

\begin{figure}
\centering
  \includegraphics[width=0.95\linewidth, clip]{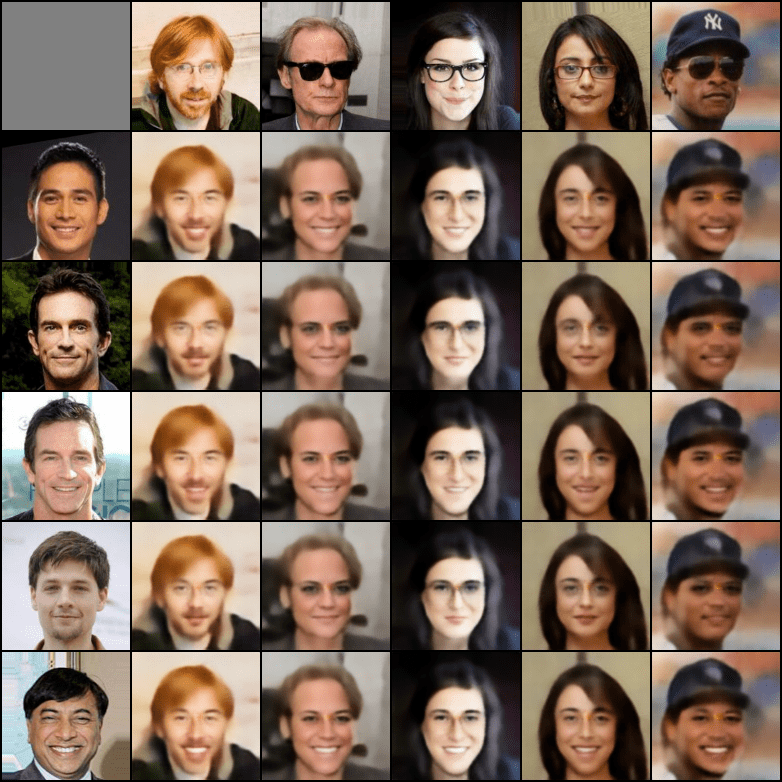} \\ 
   \caption{Translating from the domain of persons with glasses to the domain of smiling persons (reverse translation to Fig.~2 in main report)}
  \label{fig:glasses_smiling}
\end{figure}

\begin{figure}
\centering
  \includegraphics[width=0.95\linewidth, clip]{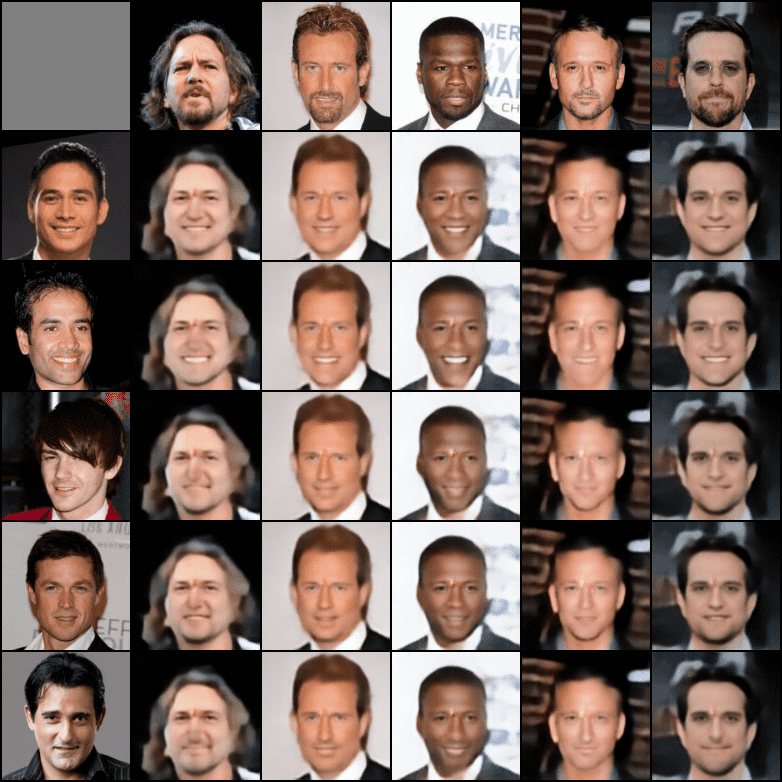} \\ 
   \caption{Translating from the domain of persons with facial hair to the domain of smiling persons.}
  \label{fig:beard_smiling}
\end{figure}

\begin{figure}
\centering
  \includegraphics[width=0.95\linewidth, clip]{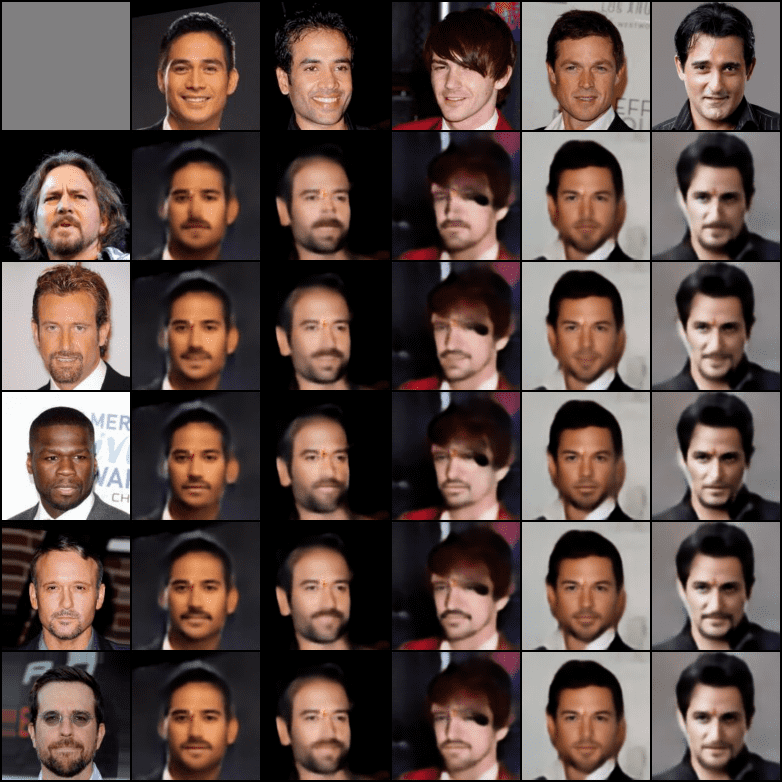} \\ 
   \caption{Reverse translation from the domain of smiling persons to the domain of persons with facial hair.}
  \label{fig:smiling_beard}
\end{figure}

\begin{figure}
\centering
  \includegraphics[width=0.95\linewidth, clip]{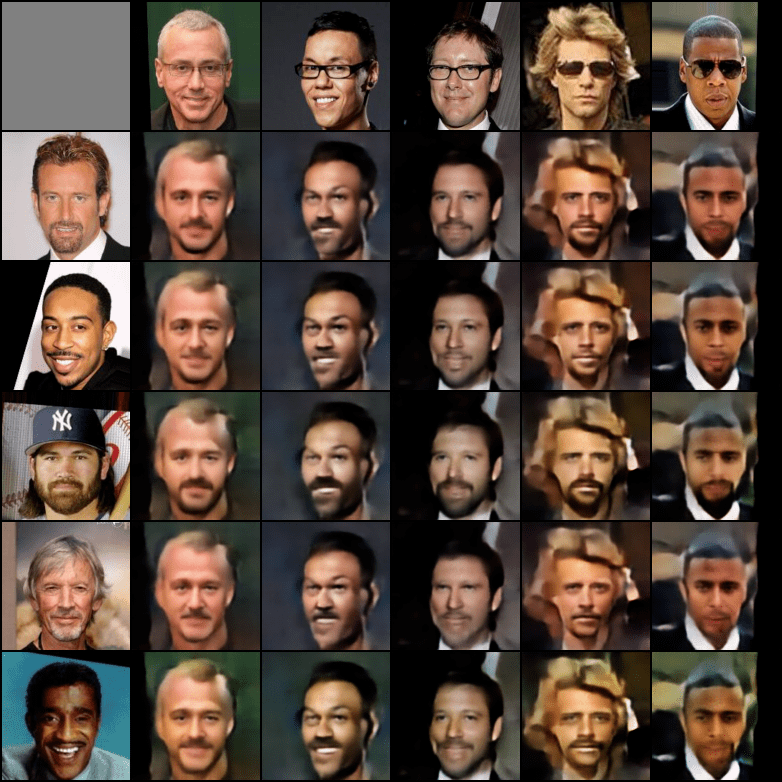} \\ 
   \caption{Translating from the domain of persons with glasses to the domain of persons with facial hair.}
  \label{fig:glasses_beard}
\end{figure}

\begin{figure}
\centering
  \includegraphics[width=0.95\linewidth, clip]{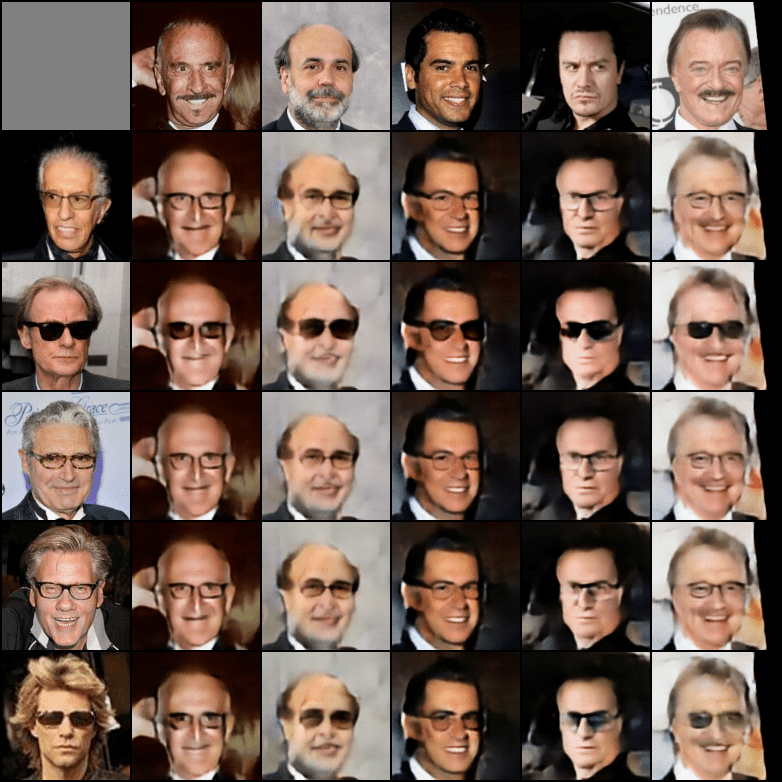} \\ 
   \caption{Reverse translation from the domain of persons with facial hair to the domain of persons with glasses.}
  \label{fig:beard_glasses}
\end{figure}

\begin{figure}
\centering
  \includegraphics[width=0.95\linewidth, clip]{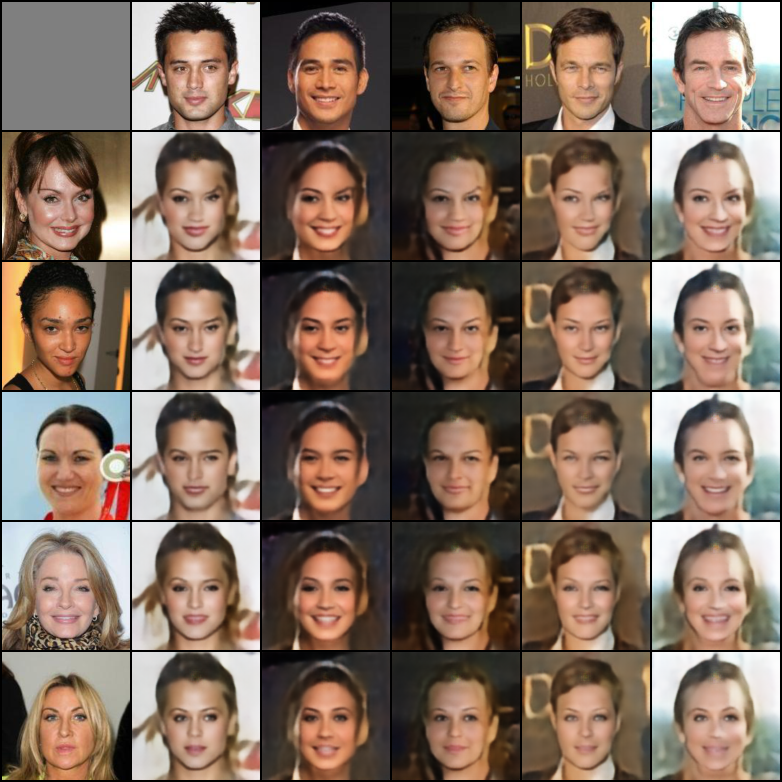} \\ 
   \caption{Translating from the domain of males to the females.}
  \label{fig:male_female}
\end{figure}

\begin{figure}
\centering
  \includegraphics[width=0.95\linewidth, clip]{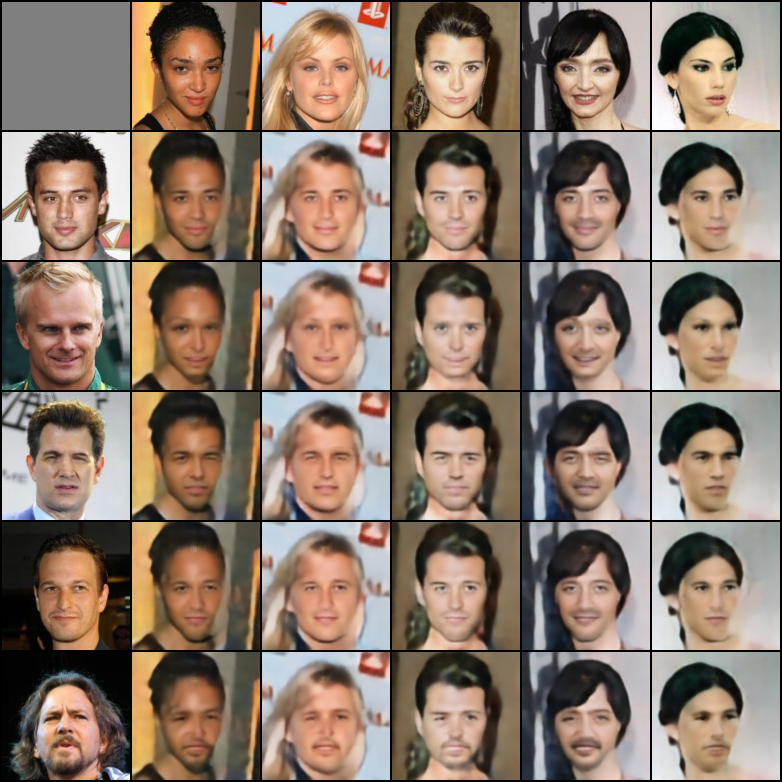}\\
   \caption{Reverse translation from the domain of females to the domain of females.}
  \label{fig:female_male}
\end{figure}

\begin{figure}
\centering
  \includegraphics[width=0.95\linewidth, clip]{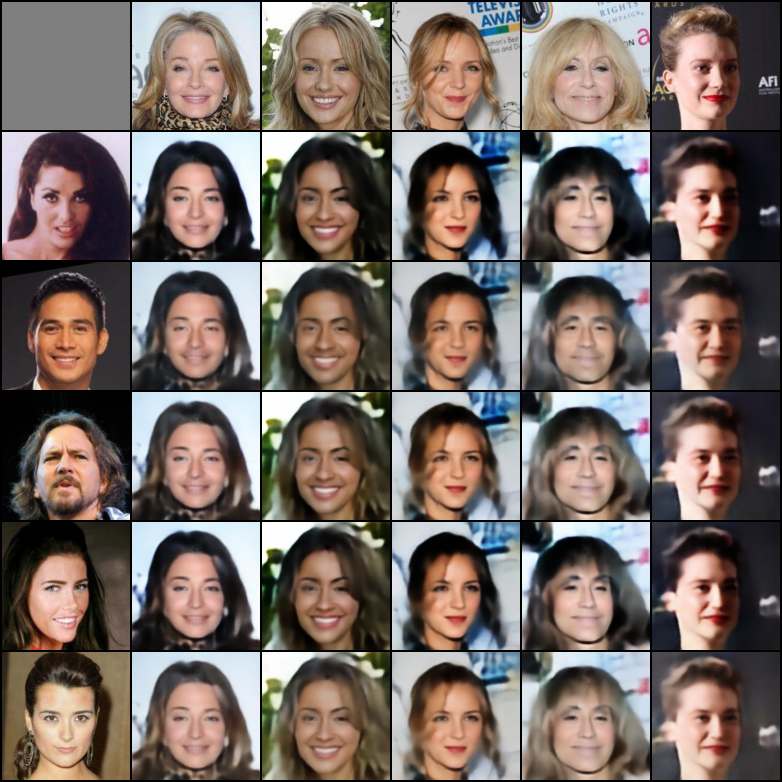}\\
   \caption{Translation from the domain of blond hair to the domain of black hair.}
  \label{fig:blond_black}
\end{figure}

\begin{figure}
\centering
  \includegraphics[width=0.95\linewidth, clip]{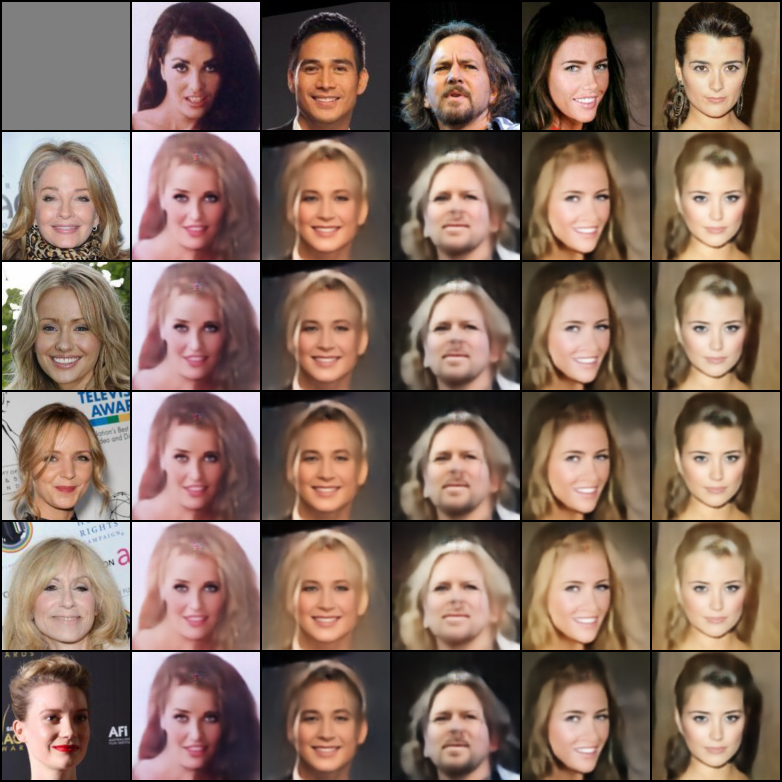}\\
   \caption{Reverse translation from the domain of black hair to the domain of blond hair.}
  \label{fig:black_blond}
\end{figure}

\section{Architecture and Hyperparameters}

We consider samples in $A$ and $B$ to be images in $\mathbb{R}^{3\times 128\times 128}$.
The encoders $E_c$, $E_s^A$ and $E_s^B$ each consist of $6$ convolutional blocks. Similarly, $G$ consists of $6$ de-convolutional blocks. 

A convolutional block $d_k$ consisting of: (a) $4\times4$ convolutional layer with stride $2$, pad $1$ and $k$ filters  (b) a spectral normalization layer (c) an instance normalization layer (d) a Leaky ReLU activation with slope $0.2$. Similarly a de-convolutional block $u_k$ consists of: (a) $4\times4$ de-convolutional layer with stride $2$, pad $1$ and $k$ filters  (b) a spectral normalization layer (c) an instance normalization layer (d) a ReLU activation.

The structure of the encoders and generators is then:
\begin{align*}
& E_c\text{: }  d_{32}, d_{64}, d_{128}, d_{256}, d_{512-sep}, d_{512-2\cdot sep} \\ 
& E_A^s, E_B^s\text{: } d_{32}, d_{64}, d_{128}, d_{128}, d_{128}, d_{sep} \\
& G\text{: } u_{512}, u_{256}, u_{128}, u_{64}, u_{32}, u^*_{3}
\end{align*}
The last layer of $G$ ($u^*_{3}$) differs in that it doesn't contain a spectral or instance normalization and that Tanh activation is applied instead of ReLU. $sep$ is the dimension of the separate encoders, set to be $25$ for all datasets. 

The latent discriminator $d$ consists of a fully connected layer of $512$ filters, a Leaky ReLU activation with slope $0.2$, second fully connected layer of $1$ filters and a final sigmoid activation. 

For the loss parameters specified in the equation 11 of the main report, $\lambda_1$ is set to $0.001$ and $\lambda_2$ to $1$. 
We use the Adam optimizer  with $\beta_1 = 0.5, \beta_2 = 0.999$, and 
learning rate of $0.0002$. We use a batch size of size $32$ in training.

\section{Theoretical Analysis}\label{sec:appanalysis}

In this section we provide a formal version of Thm.~1 from the main text. For this purpose, we recall a few technical notations from~\cite{Cover:2006:EIT:1146355}:
the Shannon entropy (discrete or continuous) $H(X) := -\mathbb{E}_{X}[ \log_2 \mathbb{P}[X]]$, the conditional entropy $H(X|Y) := H(X,Y) - H(Y)$, the (conditional) mutual information (discrete or continuous) $I(X; Y| Z) := H(X|Z) - H(X|Y,Z)$. For clarity, we list a few important identities that are being used throughout the proofs in this section. For any two random variables $X$ and $Y$, we have: $I(X;Y) = H(X) + H(Y) - H(X,Y)$. The data processing inequality, for any random variable $X$ and two functions $f$ and $g$, we have: $I(X;g(f(X))) \leq I(X;f(X))$.

In Sec.~2 in the main text, we represented our random variable $a \sim \mathbb{P}_A$ and $b \sim \mathbb{P}_B$ in the following forms $a = g(e^c(a),e^s_A(a),0)$ and $b = g(e^c(b),0,e^s_B(b))$, where $e^c(a) \indep e^s_A(a)$, $e^c(b) \indep e^s_B(b)$ and $g$ is some invertible function. Our method learns three encoders $E(x) := (E^c(x),E^s_A(x),E^s_B(x))$ and a decoder $G$. 

The following theorem is a formal version of Thm.~1 from the main text.

\begin{theorem}\label{thm:disent}
In the setting of Sec.~2 in the main text. Let $a \sim \mathbb{P}_A$ and $b \sim \mathbb{P}_B$ be two random variables distributed by discrete distributions $\mathbb{P}_A$ and $\mathbb{P}_B$. Assume that the representations $g(e^c(a),e^s_A(a),0)$ and $g(e^c(b),0,e^s_B(b))$ form an intersection between $a$ and $b$, such that,
\begin{equation}\label{eq:infoflow}
\begin{aligned}
H(E^s_A(a)) \leq H(e^s_A(a)) + \epsilon \\
\end{aligned}
\end{equation}
In addition, assume that: $\mathbb{E}_{a} \| G(E^c(a),E^s_A(a),0) - a\|_1 = 0$, $\mathbb{E}_{b} \| G(E^c(b),0,E^s_B(b)) - a\|_1 = 0$ and $\mathbb{P}_{E^c(A)} = \mathbb{P}_{E^c(B)}$, i.e., the distribution of $E^c(A)$ is equal to the distribution of $E^c(B)$. Then, we have the following:
\begin{itemize}
\item $I(E^c(a);E^s_A(a)) \leq \epsilon$.
\item $E^c(a)$ is a function of $e^c(a)$.
\item $H(E^c(a)) \geq H(e^c(a)) - \epsilon$.
\end{itemize}
\end{theorem}

In this theorem, we make a few assumptions. The first assumption concerns the modeling of the data, the second is regarding the separate encoder $E^s_A$ and the last one concerns the losses. 

% As consequences of this theorem, the common part $E^c(a)$ and the separate part $E^s_A(a)$ are (almost) independent and the common part $E^c(a)$ captures the information in $e^c(a)$.

Our first assumption asserts that the ground truth representation (see Sec.~2) of the random variables $a = g(e^c(a),e^s_A(a),0)$ and $b = g(e^c(b),0,e^s_B(b))$ forms an intersection between them. Put differently, we can partition the information of $a$ and $b$ into independent features $e^c(a)$, $e^s_A(a)$ for $a$ and $e^c(b)$, $e^s_B(b)$ for $b$, such that, the information of $e^c(a) \sim e^c(b)$ is maximal. Informally, any other partition into common and separate parts is unable to put more content information in the common part than the amount the ground truth representations do. For example, in the case where $A$ consists of images of persons with facial hair and $B$ consists of images of persons with glasses, the assumption is verified, since, we cannot transfer information from the separate part (facial hair or glasses) into the common part (identity, pose, etc'). 

The second assumption asserts that the amount of information encoded in $E^s_A(a)$ is bounded by the amount of information encoded in $e^s_A(a)$. Differently viewed, since the function $E^s_A$ is deterministic, we also have $I(E^s_A(a);a) = H(E^s_A(a))$, and therefore, the amount of mutual information between $E^s_A(a)$ and $a$ is bounded as well. This implies that we cannot recover $a$ given $E^s_A(a)$, since we cannot recover $a$ from $e^s_A(a)$. 

The third assumption is that several losses are minimized. In Sec.~3, we introduced reconstruction losses: $\mathcal{L}^A_{recon}$ and $\mathcal{L}^B_{recon}$ and an adversarial loss: $\mathcal{L}_{adv}$. These losses were measured on average with respect to the training set. In Thm.~\ref{thm:disent}, the reconstruction losses $\mathcal{L}^A_{recon}$ and $\mathcal{L}^B_{recon}$ are replaced with their expected versions (we take expectations $\mathbb{E}_a$ and $\mathbb{E}_b$ instead of averages over the training sets $\mathcal{S}_A$ and $\mathcal{S}_B$), $\mathbb{E}_{a} \| G(E^c(a),E^s_A(a),0) - a\|_1$ and $\mathbb{E}_{b} \| G(E^c(b),0,E^s_B(b)) - b\|_1$. In the theorem, we assume that these losses are being minimized by $E^c,E^s_A,E^s_B$ and $G$. In addition, the expected version of $\mathcal{L}_{adv}$ is $\sup_{d} \left\{\mathbb{E}_a l(d(E^c(a)),1) + \mathbb{E}_{b} l(d(E^c(b)),1)\right\}$ which is minimized by any encoder $E^c$ that provides $\mathbb{P}_{E^c(A)} = \mathbb{P}_{E^c(B)}$ (see Prop.~2 in~\cite{gan}), i.e., the distribution of $E^c(a)$ is equal to the distribution of $E^c(b)$. In Thm.~\ref{thm:disent}, we assume that $\mathbb{P}_{E^c(A)} = \mathbb{P}_{E^c(B)}$ which implies that the adversarial loss is minimized as well. We note that in this analysis the zero-losses are not a requirement. It is also depicted in our ablation study that the zero-losses are not a requirement but slightly improve the results. 

The consequences of the theorem are: (i) the encodings $E^c(a)$ and $E^s_A(a)$ are (almost) independent, (ii) $E^c(a)$ is a function of $e^c(a)$ and (iii) $E^c(a)$ holds most of information in $e^c(a)$. The second and third consequences provide that $E^c(a)$ and $e^c(a)$ encode the same information. We note that, given these consequences, we could also claim that $E^s_A(a)$ and $e^s_A(a)$ hold the same information. Therefore, we conclude that under the proposed assumptions, the learned encodings $E^c(a)$ and $E^s_A(a)$ capture the same information as $e^c(a)$ and $e^s_A(a)$ (resp.).

Finally, for clarity, we note that by symmetric arguments, we could arrive at the same conclusions for $E^c(b)$ and $E^s_B(b)$.

\section{Proof of Thm.~\ref{thm:disent}}

\begin{proof}[Proof of Thm.~\ref{thm:disent}]
First, we consider that by $I(X;Y) = H(X) + H(Y) - H(X,Y)$, we have:
\begin{equation}\label{eq:first}
\begin{aligned}
I(E^c(a);E^s_A(a)) = & H(E^c(a)) + H(E^s_A(a)) \\
&- H(E^c(a),E^s_A(a))
\end{aligned}
\end{equation}
Since $\mathbb{E}_{a} \| G(E^c(a),E^s_A(a),0) - a\|_1 = 0$, we have:
\begin{equation}\label{eq:press}
\begin{aligned}
&I(G(E^c(a),E^s_A(a),0);a) = I(a;a) = H(a)
\end{aligned}
\end{equation}
Next, by the data processing inequality, we have: $I(X;g(f(X))) \leq I(X;f(X))$. Therefore, by selecting $g(\cdot) := G(\cdot,0)$ and $g(\cdot) := (E^c(\cdot),E^s_A(\cdot))$ and $X := a$, we have:
\begin{equation}\label{eq:f}
\begin{aligned}
H(a) &= I(G(E^c(a),E^s_A(a),0);a) \\
&\leq I(E^c(a),E^s_A(a);a)
\end{aligned}
\end{equation}
Since $a = g(e^c(a),e^s_A(a),0)$, where $e^c(a)$ and $e^s_A(a)$ are assumed to be independent (see Sec.~2) and $g$ to be is invertible, we have:
\begin{equation}\label{eq:H}
\begin{aligned}
H(a) &= H(g(e^c(a),e^s_A(a),0)) \\
&= H(e^c(a),e^s_A(a)) \\
&= H(e^c(a)) + H(e^s_A(a)) \\
\end{aligned}
\end{equation}
We assumed that the representations $g(e^c(a),e^s_A(a),0)$ and $g(e^c(b),0,e^s_B(b))$ form an intersection between $a$ and $b$. In addition, $G(E^c(a),E^s_A(a),0) \sim \mathbb{P}_A$, $G(E^c(b),0,E^s_B(b)) \sim \mathbb{P}_B$ and $E^c(a) \sim E^c(b)$ (since we assumed that $\mathbb{P}_{E^c(A)} = \mathbb{P}_{E^c(B)}$). Therefore, for $G := \hat{g}$, $\hat{e}^c := E^c$, $\hat{e}^s_A := E^s_A$ and $\hat{e}^s_B := E^s_B$, by Def.~1:
\begin{equation}\label{eq:inter}
H(E^c(a)) \leq H(e^c(a)) \\
\end{equation}
By Eq.~\ref{eq:infoflow}, we have:
\begin{equation}\label{eq:twoeps}
\begin{aligned}
H(E^s_A(a)) - \epsilon &\leq H(e^s_A(a)) \\
\end{aligned}
\end{equation}
By combining Eqs.~\ref{eq:f},~\ref{eq:H},~\ref{eq:inter} and~\ref{eq:twoeps}, we have:
\begin{equation}
\begin{aligned}
H(E^c(a),E^s_A(a)) &\geq H(a)\\
&= H(e^c(a)) + H(e^s_A(a))\\
&\geq H(E^c(a)) + H(E^s_A(a)) - \epsilon \\
\end{aligned}
\end{equation}
By combining the last inequality with Eq.~\ref{eq:first}, we have:
\begin{equation}
\begin{aligned}
I(E^c(a);E^s_A(a)) \leq \epsilon
\end{aligned}
\end{equation}
Next, we define $\hat{e}^c(a) := (e^c(a),E^c(a))$, $\hat{e}^s_A(a) := (e^s_A(a),E^s_A(a))$, $\hat{e}^s_B(b) := (e^s_B(b),E^s_B(b))$ and $g'$, such that, $g'(\hat{e}^c(a),\hat{e}^s_A(a),0) = g(e^c(a),e^s_A(a))$ and $g'(\hat{e}^c(b),0,\hat{e}^s_B(b)) = g(e^c(b),e^s_B(b))$. Since $g$ is invertible for both domains, we conclude that $g'$ is invertible as well. Therefore, by Def.~1, we conclude that $H(\hat{e}^c(a)) \leq H(e^c(a))$. But, $\hat{e}^c(a) = (e^c(a),E^c(a))$ and, therefore, we also have: $H(\hat{e}^c(a)) \geq H(e^c(a))$. In particular, $H(\hat{e}^c(a)) = H(e^c(a))$. We conclude that:
\begin{equation}
\begin{aligned}
I(e^c(a);E^c(a)) =& H(e^c(a)) + H(E^c(a)) \\
&- H(e^c(a),E^c(a)) \\
=& H(E^c(a))
\end{aligned}
\end{equation}
Therefore, $E^c(a)$ is a function of $e^c(a)$. Finally, we consider that:
\begin{equation}
\begin{aligned}
&H(E^c(a)) + H(e^s_A(a)) + \epsilon \\
\geq& H(E^c(a)) + H(E^s_A(a)) \\
\geq& H(a) \\
=& H(e^c(a)) + H(e^s_A(a))
\end{aligned}
\end{equation}
In particular, $H(E^c(a)) \geq H(e^c(a)) - \epsilon$.
\end{proof}

\section{Ablation Study Visual Results}

In order to compare the effect of the different loss visually, we provide in Fig.~\ref{fig:no_recon}, \ref{fig:no_adv} and \ref{fig:no_zero} the translation from smiling persons to persons with glasses, when each of the losses is removed. With no reconstruction loss the method is unable to create realistic face images, as the $G$ is not affected by any of the losses remaining. With no adversarial loss the method is unable to add the glasses (separate part of domain $B$) to the given image. Without the zero-loss, results are only slightly worse numerically, and this is not observed visually. 

\begin{figure}
\centering
  \includegraphics[width=0.95\linewidth, clip]{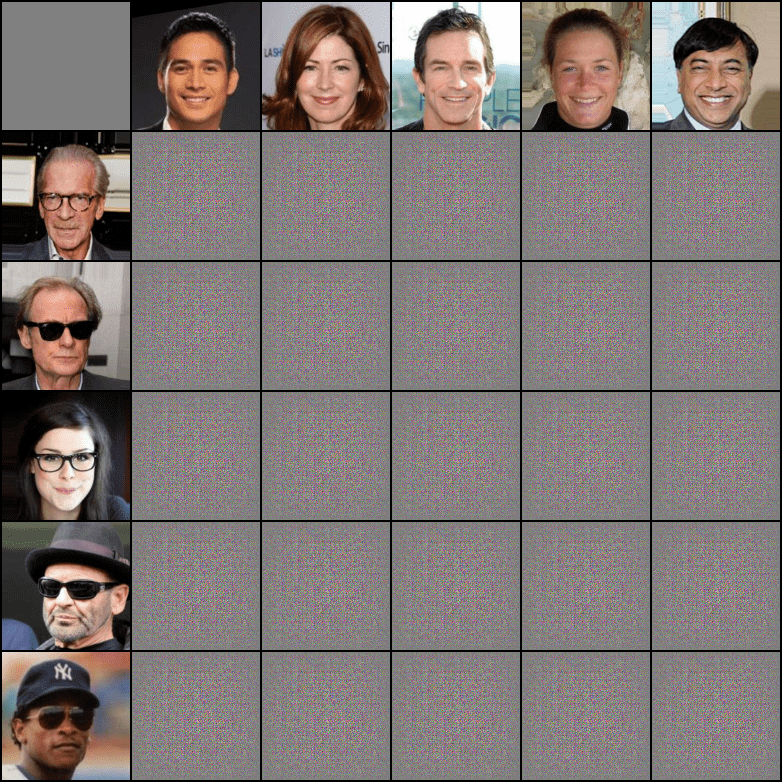}\\
   \caption{Translation from the domain of smiling persons to the domain of persons with glasses, when the reconstruction loss is removed}
  \label{fig:no_recon}
\end{figure}

\begin{figure}
\centering
  \includegraphics[width=0.95\linewidth, clip]{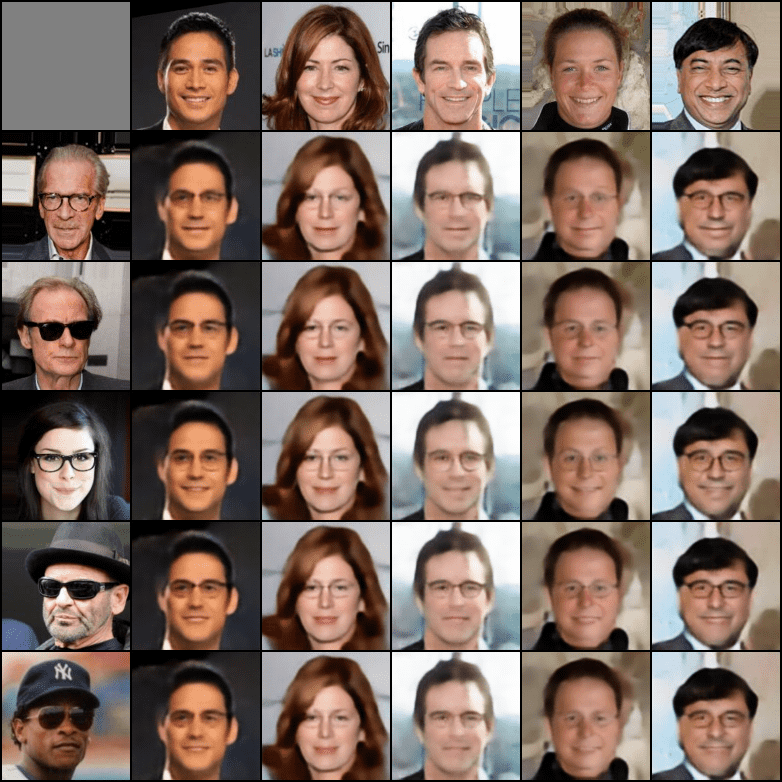}\\
   \caption{Translation from the domain of smiling persons to the domain of persons with glasses, when the adversarial loss is removed}
  \label{fig:no_adv}
\end{figure}

\begin{figure}
\centering
  \includegraphics[width=0.95\linewidth, clip]{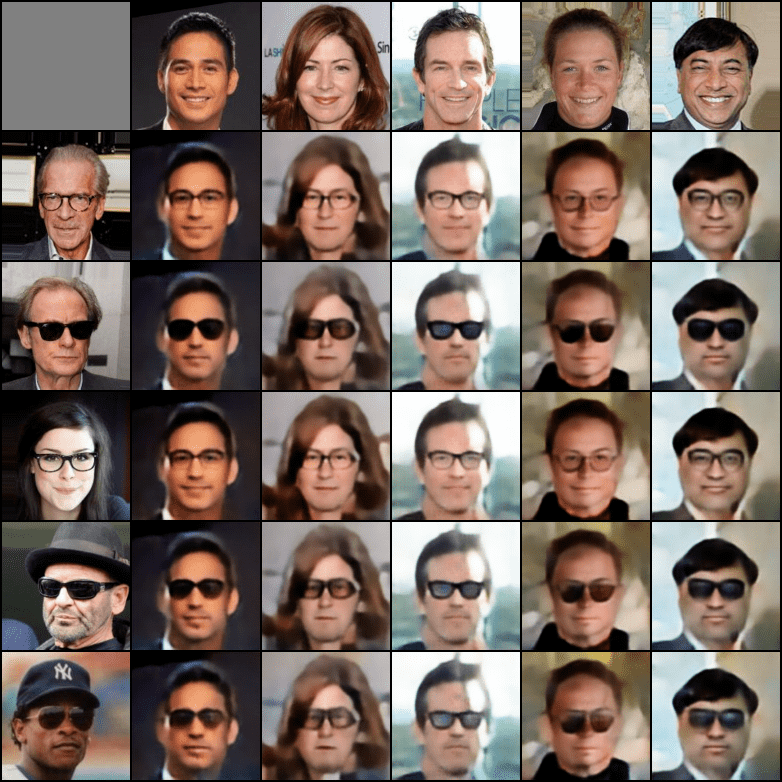}\\
   \caption{Translation from the domain of smiling persons to the domain of persons with glasses, when the zero loss is removed}
  \label{fig:no_zero}
\end{figure}

\section{Visual Comparison to Baseline Methods}

In additional to the numerical comparison in tables 1 and 2 of the main report, we provide a visual comparison in Fig.~\ref{fig:fader}, \ref{fig:drit} and \ref{fig:munit}. For MUNIT and DRIT, the method is unable to change content in the source image, and so the smile (separate part of domain $A$) remains, and no glasses (separate part of domain $B$) are added. For Fader Networks, a generic glasses are added, and not the one specific to the image in domain $B$. 

\begin{figure}
\centering
  \includegraphics[width=0.95\linewidth, clip]{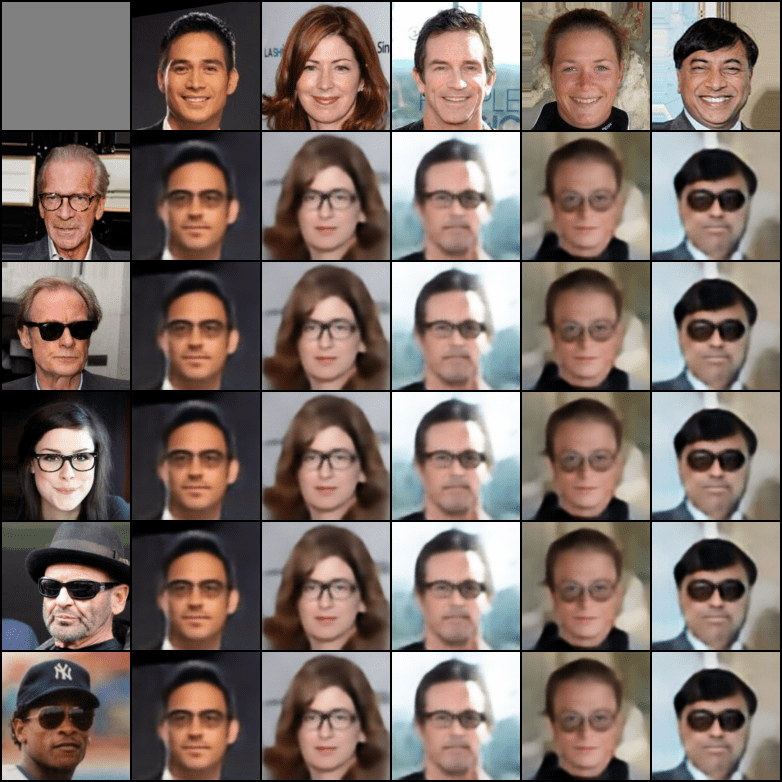}\\
   \caption{Translation from the domain of smiling persons to the domain of persons with glasses, using the Fader Networks method.}
  \label{fig:fader}
\end{figure}

\begin{figure}
\centering
  \includegraphics[width=0.95\linewidth, clip]{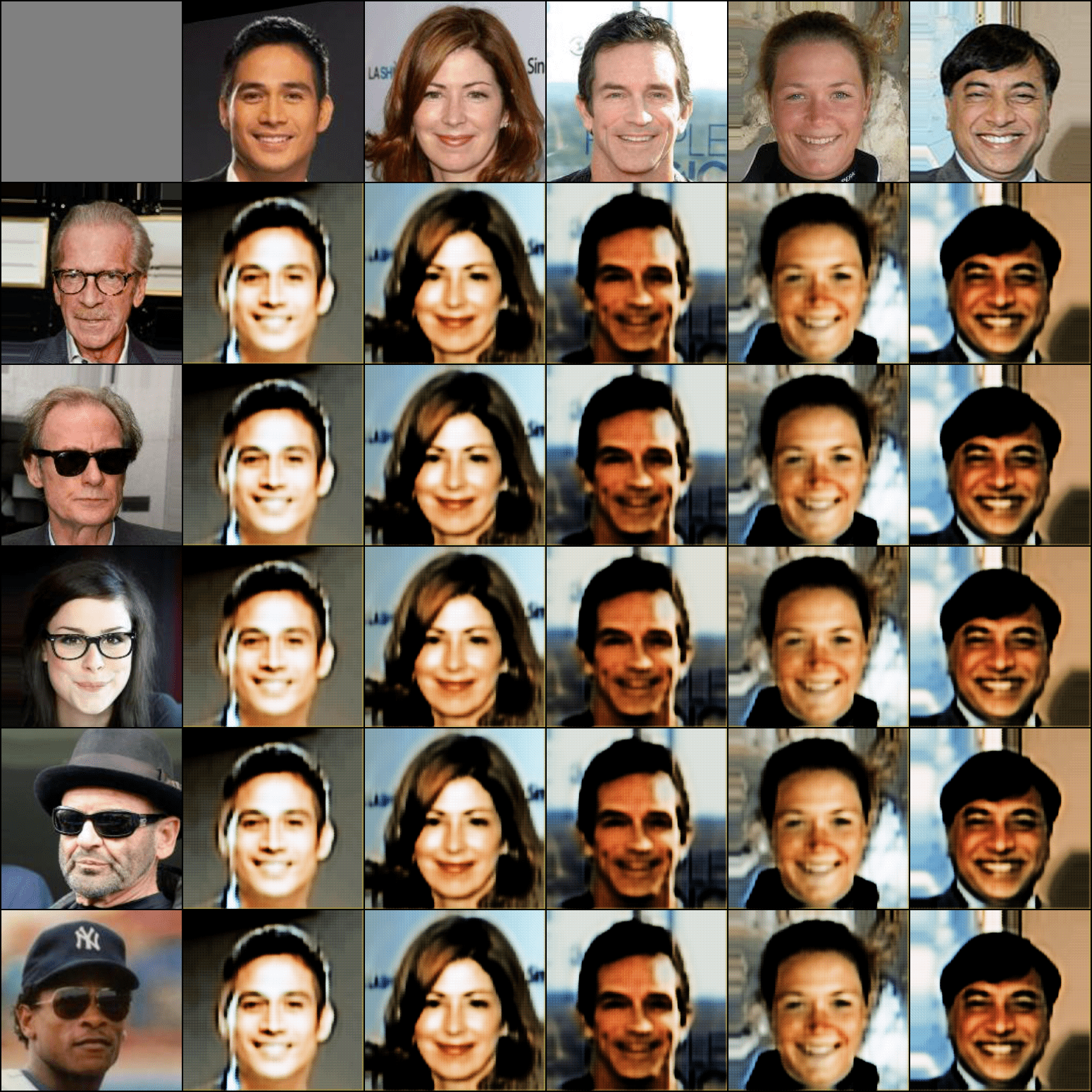}\\
   \caption{Translation from the domain of smiling persons to the domain of persons with glasses, using the DRIT method.}
  \label{fig:drit}
\end{figure}

\begin{figure}
\centering
  \includegraphics[width=0.95\linewidth, clip]{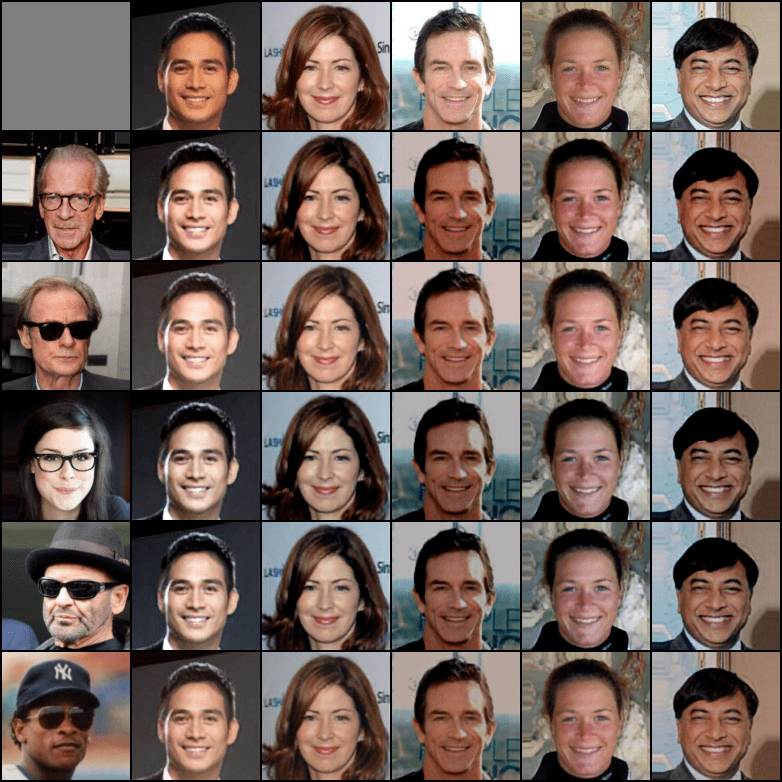}\\
   \caption{Translation from the domain of smiling persons to the domain of persons with glasses, using the MUNIT method.}
  \label{fig:munit}
\end{figure}

% \section{Code}

% Anonymized code to reproduce the results in the report is provided in a separate folder as part of the supplementary material. 

\end{document}